\documentclass{article}

\usepackage{microtype}
\usepackage{graphicx}
\usepackage{subcaption}
\usepackage{booktabs} 

\usepackage[preprint]{icml2026}

\usepackage{amsmath}
\usepackage{amssymb}
\usepackage{mathtools}
\usepackage{amsthm}

\usepackage[capitalize,noabbrev]{cleveref}

\theoremstyle{plain}
\newtheorem{theorem}{Theorem}[section]
\newtheorem{proposition}[theorem]{Proposition}
\newtheorem{lemma}[theorem]{Lemma}

\theoremstyle{definition}
\newtheorem{definition}[theorem]{Definition}
\newtheorem{assumption}[theorem]{Assumption}
\theoremstyle{remark}

\icmltitlerunning{FATE: Closed-Loop Feasibility-Aware Task Generation with Active Repair for Physically Grounded Robotic Curricula}

\begin{document}

\twocolumn[
  \icmltitle{FATE: Closed-Loop Feasibility-Aware Task Generation \\ with Active Repair for Physically Grounded Robotic Curricula}



\icmlsetsymbol{equal}{*}

\begin{icmlauthorlist}
    \icmlauthor{Bingchuan Wei}{equal,thuAero}
    \icmlauthor{Bingqi Huang}{equal,thuAero}
    \icmlauthor{Jingheng Ma}{thuInte}
    \icmlauthor{Zeyu Zhang}{BIGAI}
    \icmlauthor{Sen Cui}{thuAuto}
\end{icmlauthorlist}

\icmlaffiliation{thuAero}{School of Aerospace Engineering, Tsinghua University, Beijing, China}
\icmlaffiliation{thuAuto}{Department of Automation, Tsinghua University, Beijing, China}
\icmlaffiliation{thuInte}{School of Integrated Circuits, Tsinghua University, Beijing, China}
\icmlaffiliation{BIGAI}{State Key Laboratory of General Artificial Intelligence, Beijing Institute for General Artificial Intelligence (BIGAI), Beijing, China}

\icmlcorrespondingauthor{Sen Cui}{cuis@mail.tsinghua.edu.cn}

\vskip 0.3in
]



\printAffiliationsAndNotice{}  

\begin{abstract}
Recent breakthroughs in generative simulation have harnessed Large Language Models (LLMs) to generate diverse robotic task curricula, yet these open-loop paradigms frequently produce linguistically coherent but physically infeasible goals—stemming from ungrounded task specifications or misaligned objective formulations. To address this critical limitation, we propose FATE (Feasibility-Aware Task gEneration)—a closed-loop, self-correcting framework that reimagines task generation as an iterative validation-and-refinement process. Unlike conventional methods that decouple generation and verification into discrete stages, FATE embeds a generalist embodied agent directly into the generation loop to proactively guarantee the physical groundedness of the resulting curriculum. FATE instantiates a sequential auditing pipeline: it first validates static scene attributes (e.g., object affordances, layout compatibility) and subsequently verifies execution feasibility via simulated embodied interaction. Critical to its performance, upon detecting an infeasible task, FATE deploys an active repair module that autonomously adapts scene configurations or policy specifications—converting unworkable proposals into physically valid task instances. Extensive experiments validate that FATE generates semantically diverse, physically grounded task curricula while achieving a substantial reduction in execution failure rates relative to state-of-the-art generative baselines.
\end{abstract}

\section{Introduction}
The paradigm of embodied artificial intelligence is shifting from static environments to open-ended, procedurally generated worlds \cite{wang2023gensim, wang2023robogen}. Open-loop LLM-based generators can produce diverse task proposals, but many are infeasible due to static inconsistencies (e.g., instruction--scene mismatch, asset affordance violations, or unsatisfied preconditions). Even when the initial setup is plausible, tasks may still fail at runtime due to robot reachability limits, contact instability, or planning/learning breakdowns. Consequently, unconstrained open-loop mechanisms produce invalid data and waste computational resources.
\begin{figure*}[ht]
  \begin{center}
    \centerline{\includegraphics[width=\textwidth]{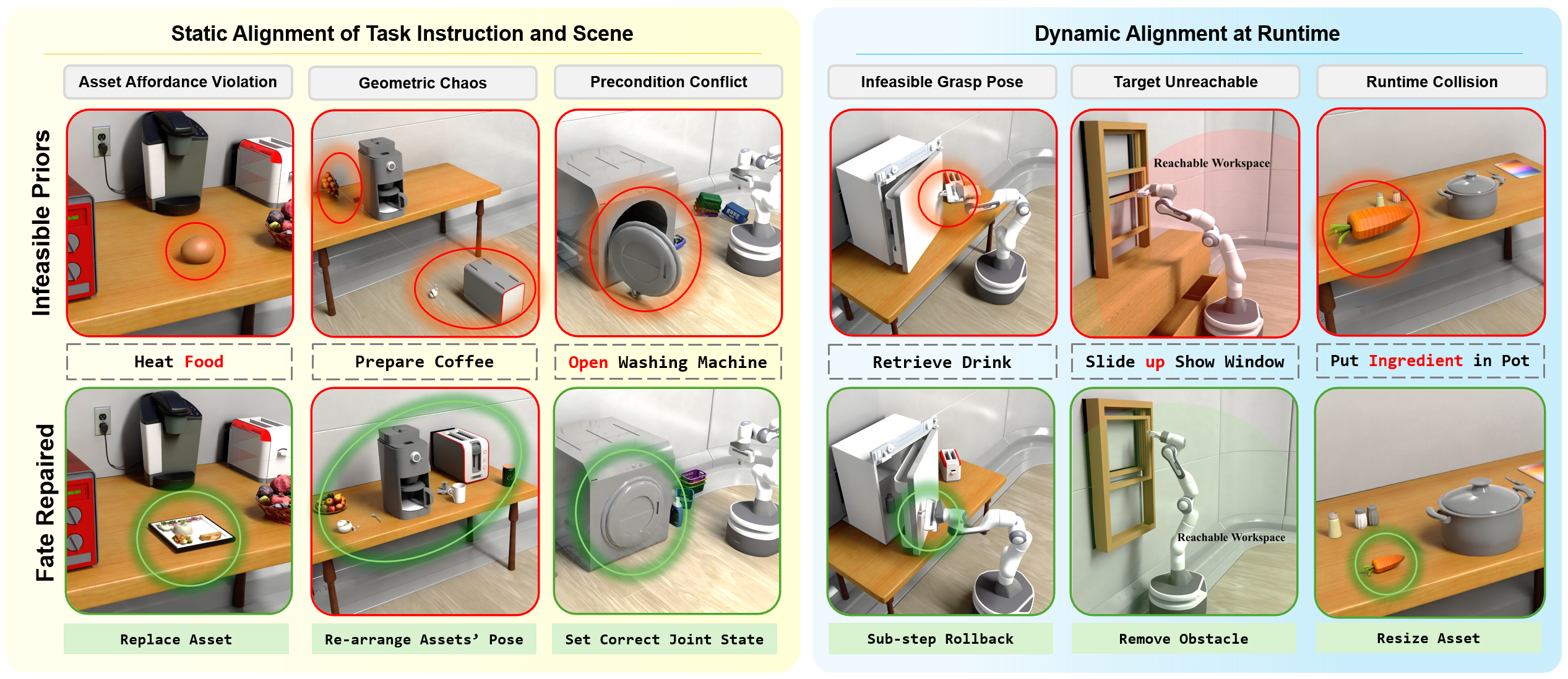}}
    \caption{\textbf{FATE aligns task feasibility across static and dynamic dimensions.} 
      \textbf{(Left) Static Alignment:} The system corrects initial scene priors by resolving semantic violations (e.g., replacing raw eggs), geometric chaos, and precondition conflicts. 
      \textbf{(Right) Dynamic Alignment:} FATE addresses runtime execution failures, such as high-torque grasps, unreachable targets, and collisions, through iterative rollback mechanisms. 
      \textbf{Red} highlights indicate initial infeasibilities; \textbf{Green} indicates FATE-repaired executable environments.}
    \label{fig:teaser}
  \end{center}
\end{figure*}

Building on the aforementioned observations, we argue that task specification rationality (i.e., the alignment of task settings with real-world commonsense) and task objective executability (i.e., the practical feasibility of goal achievement via robotic interaction) must be fully incorporated into a closed-loop task generation pipeline. This integration ensures that generated tasks not only maintain linguistic coherence, but also adhere to physical grounding—addressing the fundamental limitation of open-loop generation. 

To overcome this challenge, we propose to treat feasibility as a hierarchical constraint satisfaction problem spanning static perception and runtime execution. Specifically, visual plausibility does not imply solvability; a coherent scene may fail due to unreachable target or inconspicuous collisions. Therefore, we address two distinct failure modes shown in Figure~\ref{fig:teaser}: \textbf{static hallucinations} (geometric/semantic violations) and \textbf{dynamic divergences} (runtime kinematic or interaction failures). 

We advance \textbf{FATE} (Feasibility-Aware Task gEneration), brain-in-the-loop framework for hierarchical feasibility alignment with a decoupled strategy: first projecting proposals onto a statically valid manifold, then refining the task via in-execution feedback. Specifically, first, a \textbf{static alignment} operator rectifies perceptual violations (e.g., interpenetrations). Subsequently, a \textbf{dynamic alignment} auditor monitors execution, iteratively adjusts solver parameters to resolve runtime failures. This transforms task generation from stochastic sampling into closed-loop curriculum.

We build FATE using a generative simulator and a generalist embodied foundation model, and evaluate the framework on a new benchmark designed to stress-test feasibility-aware generation. We validate FATE along three axes: \textbf{generative diversity}, \textbf{feasibility awareness}, and \textbf{refinement efficacy}. Experiments show that FATE maintains strong linguistic and visual diversity compared to expert-designed benchmarks and open-loop generators, while substantially reducing the end-to-end feasibility gap via hierarchical audit-and-repair. We further report the auditor's diagnostic accuracy and quantify how often infeasible proposals are salvaged by repair across stages. The contributions of this work are threefold:
\begin{enumerate}
    \setlength{\itemsep}{0pt}
    \setlength{\parskip}{0pt}
    \setlength{\parsep}{0pt}
    \item We propose a new problem setting, \textbf{Feasibility-Aware Open-Ended Task Generation}, that proactively accounts for the feasibility gap induced by both unjustified task settings and infeasible task goals.
    \item We propose a brain-in-the-loop framework \textbf{FATE}, which integrates dual-phase auditing and repair and proactively diagnoses failure modes via an iterative refinement process to guarantee practical feasibility.
    \item Experiments demonstrate that \textbf{FATE} can effectively identify and correct infeasible tasks, reducing simulation waste at every stage of the process while significantly boosting the performance of downstream policy learning.
\end{enumerate}


\section{Related Work}
\label{sec:related_work}
The use of LLMs to procedurally generate simulation environments has evolved significantly. Early systems like \textbf{GenSim} \cite{wang2023gensim} and \textbf{RoboGen} \cite{wang2023robogen} demonstrated the potential of generating simulation codes from text. Recent works have pushed this frontier: \textbf{GenSim2} \cite{gensim2_2025} advances this paradigm by integrating multi-modal LLMs (e.g., GPT-4V) for better scene understanding and reasoning; \textbf{AgentGen} \cite{agentgen_2025} introduces a bidirectional evolutionary strategy to generate tasks with a smooth difficulty gradient; \textbf{ReGen} \cite{regen_2025} proposes an "inverse design" paradigm, inferring simulation environments from desired agent behaviors; and \textbf{Factorsim} \cite{factorsim_2024} employs factorized POMDP representations to generate modular simulation codes. 

Despite these advances in structural coherence and difficulty scaling, these methods largely operate in a "feed-forward" manner. They lack an active \textit{feasibility auditor} to detect when generated constraints (e.g., target positions, object masses) violate the robot's physical limits. While \textbf{DrEureka} \cite{ma2024dreureka} automates reward tuning and domain randomization for sim-to-real transfer, it assumes the task itself is valid. FATE addresses this limitation by introducing an adaptive repair module that actively modifies the task definition to ensure physical executability prior to training.

A parallel line of research focuses on amplifying scarce human demonstrations. \textbf{MimicGen} \cite{mimicgen_2023} and its recent extensions \textbf{DemoGen} \cite{demogen_2025} and \textbf{SkillMimicGen} \cite{garrett2024skillmimicgen} leverage data augmentation (e.g., rigid body transformations, skill segmentation) to expand a limited set of source demos into large datasets. \textbf{DexMimicGen} \cite{jiang2025dexmimicgen} further extends this to bimanual dexterous manipulation. 
While highly effective, these methods are \textit{demonstration-dependent}—they can only generate variations of tasks that have already been demonstrated by humans. In contrast, FATE is \textit{demonstration-free} and fully open-ended. It targets the discovery of entirely novel task semantics (e.g., "use the toaster to heat the mug"—a task a human may never have demonstrated) by synthesizing the task, environment, and solution policy from scratch, employing the MLLM auditor as a scalable proxy for human verification.

The rise of video generation has led to methods that use synthesized videos as training supervision. \textbf{DreamGen} \cite{dreamgen_2025} fine-tunes video world models to generate future frames for robot learning; \textbf{Gen2Real} \cite{gen2real_2025} and \textbf{LuciBot} \cite{lucibot_2025} extract policies directly from generated videos; and \textbf{IRASim} \cite{zhu2025irasim} achieves fine-grained interaction modeling via diffusion transformers. 
However, generative video models are prone to "physics hallucinations"—generating visually plausible but physically impossible transitions \cite{vid2world2025}. FATE treats these generative models not as ground truth, but as hypotheses. By placing an embodied brain (Auditor) in the loop, FATE validates the physical consistency of generated goals (akin to the reasoning in \textbf{SEEA-R1} \cite{seear1_2025} and \textbf{MindJourney} \cite{mindjourney_2025}) and repairs them when they violate physical laws, ensuring that the "imagined" curriculum is semantically meaningful, physically grounded and learnable.

\section{Problem Formulation}
\label{sec:problem_formulation}
In this section, we start with the necessary notation and finally present the formal definition of the overall problem.


\subsection{Feasibility-Aware Task}
\label{subsec:task_def}

We posit that a robotic task is not merely a semantic instruction but a coupled system comprising linguistic intent, physical environment, and a constrained solution space.
\begin{definition}[Feasibility-Aware Task]
\label{def:task_tuple}
A task instance is defined as a tuple $\tau = (\mathcal{I}, \mathcal{S}, \Omega_\pi)$, where $\mathcal{I}$ represents the semantic-temporal specification, $\mathcal{S}$ denotes the physical scene configuration, and $\Omega_\pi$ represents the constrained execution policy manifold.
\end{definition}
\paragraph{Instruction $\mathcal{I}$ as Semantic-Temporal Specification.}
The instruction $\mathcal{I}$ is formalized as a pair $(\mathcal{L}, \Psi_\mathcal{I})$. Here, $\mathcal{L} \in \mathcal{V}^*$ denotes a sequence of natural language tokens providing high-level intent. Crucially, $\mathcal{L}$ maps to an implicit, often non-differentiable, semantic predicate functional $\Psi_\mathcal{I}: \Xi \to \{0, 1\}$, where $\Xi$ is the space of state-action trajectories. A trajectory $\xi \in \Xi$ satisfies the instruction if and only if $\Psi_\mathcal{I}(\xi) = 1$. This encapsulates the temporal logic constraints of the task without assuming a specific decomposition.
\paragraph{Scene $\mathcal{S}$ as Semantic Articulated Configuration.}
The scene $\mathcal{S}$ is defined as a set of semantic articulated entities $\mathcal{S} = \{ (M_j, q_j, \Lambda_j) \}_{j=1}^N$. Each entity consists of its physical manifold and inertial properties $M_j$, its kinematic configuration state $q_j \in \mathcal{Q}$ (spanning the Lie group $SE(3)$ and joint space $\mathbb{R}^{d_j}$), and a semantic affordance mapping $\Lambda_j: \mathcal{M}_j \to \mathbb{L}_{sem}$. The mapping $\Lambda_j$ grounds geometric features to semantic labels (e.g., graspability, articulation types), providing the physical support necessary.
\paragraph{Manifold $\Omega_\pi$ as Constrained Policies.}
Unlike standard reinforcement learning which explores the universal parameter space $\Theta$, generative simulation imposes structural constraints on the solution search. We define $\Omega_\pi \subset \Theta$ as a compact sub-manifold induced by the generator's code specification (e.g., specific reward structures, primitive selections, or planning horizons). $\Omega_\pi$ effectively delineates the \textit{permissible search space} for the agent's learning algorithm. By including $\Omega_\pi$ in the task definition, we formulate the generated instance not merely as a problem to be solved, but as a \textit{self-contained learning unit}. This formulation posits that a valid open-ended task must intrinsically encode its own solvability boundaries. 

\subsection{Feasible Policy and Task Feasibility}
\label{subsec:feasible_policy}

Validity in embodied tasks is not binary but measure-theoretic. A task is only practically feasible if the set of solutions has non-negligible volume under the agent's exploration capabilities.

\begin{definition}[Feasible Policy Set]
\label{def:feasible_set}
For a given task $\tau$, the feasible policy set $\Phi_\tau \subset \Omega_\pi$ is the subset of policies that satisfy the semantic predicate $\Psi_\mathcal{I}$ with high probability under the scene dynamics induced by $\mathcal{S}$:
\begin{equation}
    \Phi_\tau = \left\{ \theta \in \Omega_\pi \mid \mathbb{E}_{\xi \sim (\pi_\theta, \mathcal{S})} [\Psi_\mathcal{I}(\xi)] \ge 1 - \epsilon \right\}
\end{equation}
where $\epsilon$ is a tolerance parameter for execution stochasticity and $\pi_\theta$ is a policy parameterized by $\theta$.
\end{definition}
To quantify the tractability of finding such a policy, we model the policy search process as sampling from an initialization density. Let $\rho_{init}$ be a probability measure over the parameter manifold $\Omega_\pi$. This measure represents the static exploration volume accessible to the optimization algorithm.
\begin{definition}[Task Feasibility Measure]
\label{def:feasibility_measure}
The feasibility $\mu(\tau)$ is defined as the measure of the feasible set under $\rho_{init}$:
\begin{equation}
    \mu(\tau) = \int_{\Omega_\pi} \mathbb{I}(\theta \in \Phi_\tau) \, d\rho_{init}(\theta)
\end{equation}
\end{definition}
\paragraph{Implication.} A task is considered \textit{physically grounded} if $\mu(\tau) > \delta_{min}$. Conversely, if $\mu(\tau) \approx 0$, the task exhibits a \textit{Feasibility Gap}. Practically, this manifests as execution failures where the robotic RL agent never or hardly converge despite exhaustive training, either due to physical impossibility ($\Phi_\tau = \emptyset$) or because the generated policy manifold $\Omega_\pi$ fails to intersect with the solution space. This leads to an infeasible generated task.

\subsection{Problem Statement: Feasibility-Aware Task Generation}
\label{subsec:problem_statement}

We model the generative process as sampling from a task distribution $P_{\mathcal{G}}$ generated by a task generator $\mathcal{G}$ over the task space. The core challenge is that the support of the high-diversity generator $P_{\mathcal{G}}$, which consists of unverified task proposals, typically extends far beyond the physically grounded region where $\mu(\tau) > \delta_{min}$.

\begin{definition}[Feasibility Gap]
\label{def:feasibility_gap}
The \textbf{feasibility gap} is defined as the probability mass of generated tasks whose feasibility score fails to meet a minimum threshold. Given any refinement mapping $\mathcal{A}: \tau \mapsto \tau$, we define
\begin{equation}
    L_{\mathcal{G}}(\mathcal{A}) \triangleq \mathbb{E}_{\tau \sim P_{\mathcal{G}}} \left[ \mathbb{I}(\mu(\mathcal{A}(\tau)) \le \delta_{min}) \right].
\end{equation}
In particular, the \textit{pre-alignment} feasibility gap corresponds to the identity mapping $\mathcal{A} = \mathrm{Id}$.
\end{definition}

This definition provides an operational notion of \textit{executability}: under the induced policy manifold $\Omega_\pi$ and initialization measure $\rho_{init}$, a task is executable only if its feasible policy set has non-negligible measure (i.e., $\mu(\tau) > \delta_{min}$).

\paragraph{Goal of FATE.} Consequently, \textbf{FATE} is not merely to sample from $P_{\mathcal{G}}$, but to construct a refinement mechanism that bridges this gap without collapsing the generative diversity or deviating from the user's intent. Formally, we seek a mapping $\mathcal{A}: \tau \to \tau$ that projects task hypotheses into the feasible region. The optimal operator is the solution to the following constrained minimization problem
\begin{subequations}
\label{eq:FATE_optimization_target_func}
\begin{align}
    \min_{\mathcal{A}} \quad & L_{\mathcal{G}}(\mathcal{A}) \label{eq:opt_objective} \\
    \textrm{s.t.} \quad & \mathbb{E}_{\tau \sim P_{\mathcal{G}}} [D(\tau, \mathcal{A}(\tau))] \le \epsilon_{sem} \label{eq:opt_constraint_sem},
\end{align}
\end{subequations}
where Eq.\eqref{eq:opt_objective} minimizes the feasibility gap (Definition~\ref{def:feasibility_gap}). $D:\mathcal{T}\times\mathcal{T}\to\mathbb{R}_{\ge 0}$ denotes a semantic distance metric between two task instances, with smaller values indicating higher semantic consistency. Therefore, Eq.\eqref{eq:opt_constraint_sem} enforces that the projected task $\mathcal{A}(\tau)$ remains semantically congruent with the original intent within a tolerance $\epsilon_{sem}$.

\section{Methodology}
\label{sec:methodology}
In this section, we present the constructive realization of the optimization objective defined in Eq.\eqref{eq:FATE_optimization_target_func}. 


\begin{figure}[ht]
  \begin{center}
    \includegraphics[width=0.4\textwidth]{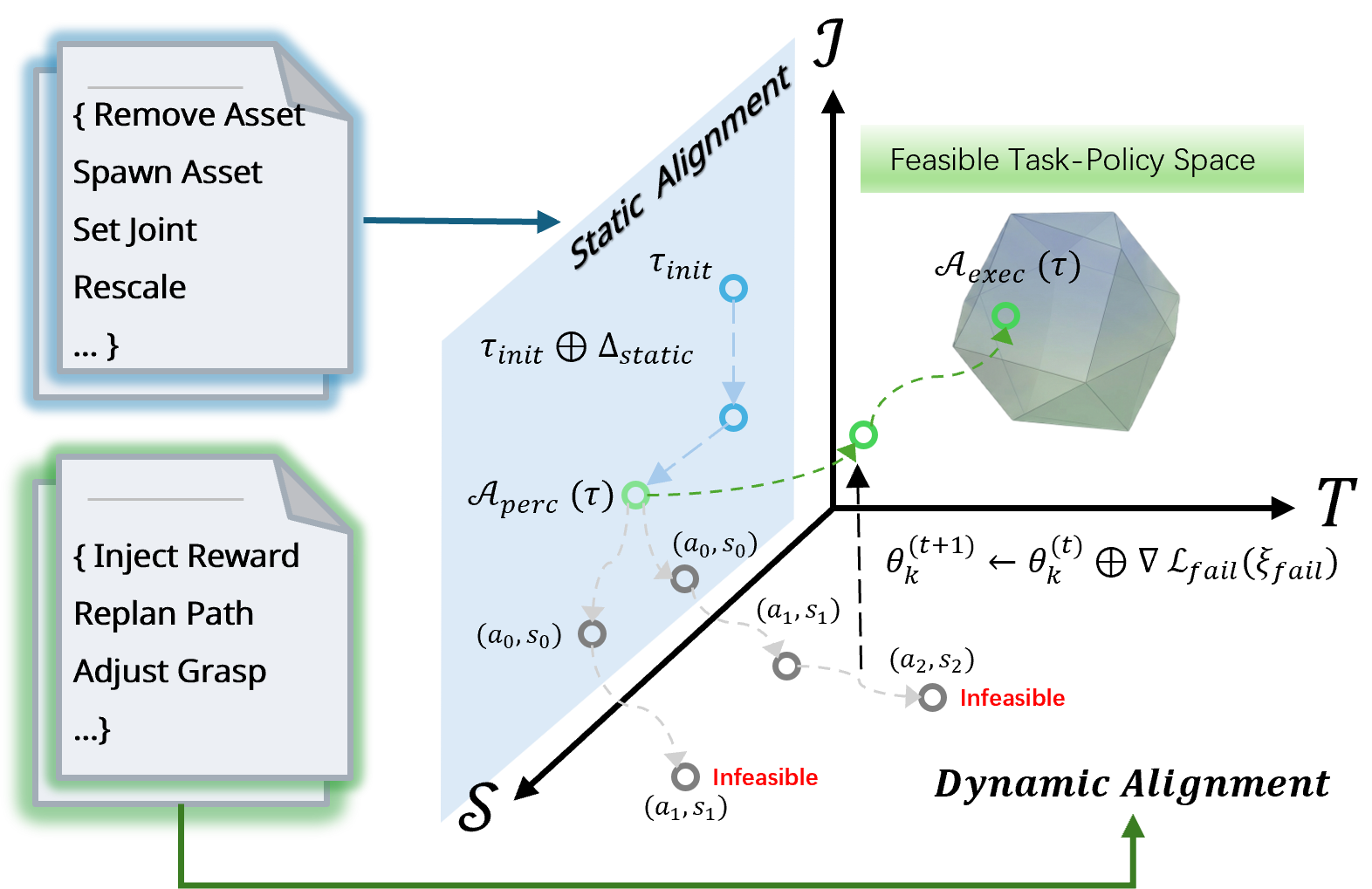}
    \caption{Visualization of the Hierarchical Feasibility Alignment Process. The optimization trajectory evolves in the joint task space $\tau = (\mathcal{I}, \mathcal{S}, \Pi)$. 
  \textbf{(Left) Static Alignment ($\mathcal{A}_{perc}$):} The Ante-Auditor corrects \textbf{Scene ($\mathcal{S}$)} and \textbf{Instruction ($\mathcal{I}$)} incompatibilities via discrete repairs (e.g., \texttt{RESCALE}), projecting $\tau_{init}$ onto the perceptually valid manifold. 
  \textbf{(Right) Dynamic Alignment ($\mathcal{A}_{exec}$):} Targeting the \textbf{Policy ($\Pi$)}, the system refines the solver configuration using semantic feedback and dynamic APIs (e.g., \texttt{INJECT REWARD}), guiding the task from failure (red trajectories) to the \textbf{Feasible Space} (green polyhedron).}
    \label{system_overview}
  \end{center}
\end{figure}
\subsection{Feasibility Alignment Operator}
\label{subsec:alignment_operator}
To bridge the Feasibility Gap, we formally construct a mapping $\mathcal{A}: \mathcal{T} \to \mathcal{T}$, termed the \textbf{Feasibility Alignment Operator}. Unlike stochastic resampling, this operator functions as a constructive manifold projection, transforming an ungrounded hypothesis $\tau \sim P_{\mathcal{G}}$ into a physically valid instance $\tau^*$ while minimizing semantic deviation:
\begin{equation}
    \tau^* = \mathcal{A}(\tau) \triangleq \underset{\mu(\tau') > \delta_{min}}{\arg\min} \ D(\tau', \tau).
\end{equation}
Recognizing the hierarchical nature of physical constraints, we structurally decompose $\mathcal{A}$ into a sequential composition of Perceptual ($\mathcal{A}_{perc}$) and Execution ($\mathcal{A}_{exec}$) sub-operators. This enforces that the task is first mapped onto the geometrically valid subspace, and subsequently projected onto the dynamic intersection.
\begin{equation}
    \mathcal{A} = \mathcal{A}_{exec} \circ \mathcal{A}_{perc}
\end{equation}

We formulate the alignment operator $\mathcal{A}$ as a \textbf{Constructive Residual Injection}. Rather than stochastic resampling, our objective is to synthesize a task residual $\Delta\tau$ that projects the current task $\tau$ into the feasible region defined by the measure $\mu(\tau)$ (Def.~\ref{def:feasibility_measure}).

Formally, at each iteration, the operator seeks an optimal update $\Delta\tau^*$ that maximizes the feasibility likelihood:
\begin{equation}
    \Delta\tau^* = \operatorname*{arg\,max}_{\Delta\tau \in \mathcal{D}} \ \mu(\tau \oplus \Delta\tau)
\end{equation}
where $\oplus$ denotes the symbolic application of the residual, and $\mathcal{D}$ is the space of permissible edits. Since computing the exact $\mu(\tau)$ is intractable, we instantiate this objective using stage-specific surrogates: maximizing the \textbf{Static Validity Score} $\hat{\mu}_{static}$ in Phase I (Sec.~\ref{subsec:static_alignment}), and minimizing the \textbf{Execution Failure Rate} in Phase II (Sec.~\ref{subsec:dynamic_alignment}).

\subsection{Policy Specification via Execution Primitives}
\label{subsec:policy_specification}

To instantiate the constrained policy manifold $\Omega_\pi$, FATE decomposes the global instruction $\mathcal{I}$ into a sequence of atomic primitives $\Pi = \{ \pi_1, \dots, \pi_K \}$. We employ a hybrid solver strategy to address the distinct topological requirements of navigation and manipulation phases:

\paragraph{Geometry-Aware Navigation (MPC).}
For non-contact phases (e.g., moving to a target), we utilize Model Predictive Control. The solver optimizes a trajectory to minimize the distance to sub-goals while strictly respecting collision constraints derived from the scene $\mathcal{S}$. This solver exposes geometric parameters (e.g., obstacle inflation radius) that determine the conservatism of the motion.

\paragraph{Contact-Rich Manipulation (RL).}
For interaction phases (e.g., grasping or opening), we employ Soft Actor-Critic (SAC). To facilitate learning in sparse-reward settings, the MLLM synthesizes a potential-based shaping function parameterized by weights $\theta_{rew}$. This allows the system to guide the agent toward semantic keypoints by shaping the optimization landscape.

\paragraph{Auditor Interface.}
Crucially, these primitives are not static black boxes but parameterized modules $\pi(\cdot \mid \theta)$. They expose a tunable configuration space (including MPC constraints and RL reward weights) which serves as the control interface for the \textbf{Dynamic Alignment Operator} ($\mathcal{A}_{exec}$). This design enables the Auditor to resolve execution failures by actively tuning the solver's logic—effectively performing a semantic policy gradient update—rather than merely resampling the task.

\begin{algorithm}[h]
   \caption{FATE: Feasibility-Aware Task Generation}
   \label{alg:fate_final}
\begin{algorithmic}[1]
   \STATE {\bfseries Input:} Random Seed $s$, Asset Library $\mathcal{L}$
   \STATE {\bfseries Output:} Feasible Task $\tau = (\mathcal{I}, \mathcal{S}, \Omega_{\pi})$
   
   \vspace{0.3em}
   \STATE \textcolor{gray}{\# 1. Initialization: Construct $\tau_{init}$}
   \STATE $\mathcal{I} \leftarrow \text{SampleInstruction}(\mathcal{L})$ 
   \STATE $\mathcal{S} \leftarrow \text{InitScene}(\text{Default Pose}, \text{Raw Scale})$ 
   \STATE $\Omega_{\pi} \leftarrow \text{InitPolicySpec}(\text{Primitives } \Pi = \{\pi_1, \dots, \pi_K\})$ 
   \STATE $\tau \leftarrow (\mathcal{I}, \mathcal{S}, \Omega_{\pi})$ \COMMENT{Initial Feasibility: Unknown}
   
   \vspace{0.3em}
   \STATE \textcolor{gray}{\# 2. Main Loop: Interleaved Alignment}
   \FOR{$k = 1, \dots, \text{MaxIter}$}
       
       \vspace{0.2em}
       \STATE \textcolor{blue}{\# Phase A: Static Alignment (Targeting Scene $\mathcal{S}$)}
       \STATE $valid, info \leftarrow \text{AnteAuditor}(\mathcal{I}, \mathcal{S})$ 
       
       \IF{\textbf{not} $valid$}
           \STATE \COMMENT{Repair geometric \& semantic violations}
           \STATE $\mathcal{S} \leftarrow \text{StaticRepair}(\mathcal{S}, info)$ 
           \STATE $\tau \leftarrow (\mathcal{I}, \mathcal{S}, \Omega_{\pi})$
           \STATE \textbf{continue} 
       \ENDIF

       \vspace{0.2em}
       \STATE \textcolor{blue}{\# Phase B: Dynamic Alignment (Targeting Policy $\Omega_{\pi}$)}
       \STATE $trace \leftarrow \text{RunSimulation}(\tau)$ 
       \STATE $success, info \leftarrow \text{InStepAuditor}(trace)$
       
       \IF{\textbf{not} $success$}
           \STATE \COMMENT{Refine solver parameters \& primitive configs}
           \STATE $\Omega_{\pi} \leftarrow \text{DynamicRepair}(\Omega_{\pi}, info)$
           \STATE $\tau \leftarrow (\mathcal{I}, \mathcal{S}, \Omega_{\pi})$
           \STATE \textbf{continue} 
       \ENDIF
       
       \vspace{0.2em}
       \STATE \textcolor{green!50!black}{\# Feasibility Confirmed}
       \STATE \textbf{return} $\tau$ 
       
   \ENDFOR
   
   \STATE \textbf{return} Failure
\end{algorithmic}
\end{algorithm}

\subsection{Phase I: Static Alignment via Perceptual Projection}
\label{subsec:static_alignment}

The first stage, $\mathcal{A}_{perc}$, acts as a perceptual gatekeeper that projects the initial task hypothesis onto the perceptually valid region. We postulate that static consistency—the alignment between the semantic instruction $\mathcal{I}$ and the scene configuration $\mathcal{S}$—is a strictly necessary condition for dynamic executability.

\paragraph{Auditing: Detecting Hallucinations.}
To estimate the feasibility measure $\mu(\tau)$ (Def.~\ref{def:feasibility_measure}) without expensive simulation, the Ante-Auditor employs a Static Feasibility Check $\hat{\mu}_{static}$. This acts as a logical filter that verifies four consistency conditions: (i) reachability; (ii) affordance matching; (iii) physical plausibility; and (iv) morphological compatibility. A failure in any condition implies the task is fundamentally impossible ($\mu(\tau) \approx 0$), immediately triggering the repair process.

\paragraph{Repair: Scene Projection.}
Upon detecting a violation, the operator synthesizes a static residual $\Delta\tau_{static}$ to rectify the scene or instruction. We define the optimal static repair as the one that maximizes the perceptual validity proxy $\hat{\mu}_{static}$:
\begin{equation}
    \label{eq:static_opt}
    \Delta\tau_{static}^* = \operatorname*{arg\,max}_{\Delta\tau \in \mathcal{D}_{static}} \ \hat{\mu}_{static}(\tau \oplus \Delta\tau)
\end{equation}
The update $\tau \leftarrow \tau \oplus \Delta\tau_{static}^*$ corresponds to a discrete projection onto the geometrically valid manifold. In practice, $\Delta\tau_{static}$ is composed of two parts: \textbf{Scene Reconfiguration} (e.g., rescaling an asset $M_j \to M_j'$ or moving an unreachable object) and \textbf{Instruction Refinement} (e.g., relaxing an over-specific verb). By solving \eqref{eq:static_opt}, $\mathcal{A}_{perc}$ ensures the proposal is physically plausible before simulation.

\begin{figure*}[ht]
    \centering
    \includegraphics[width=\linewidth]{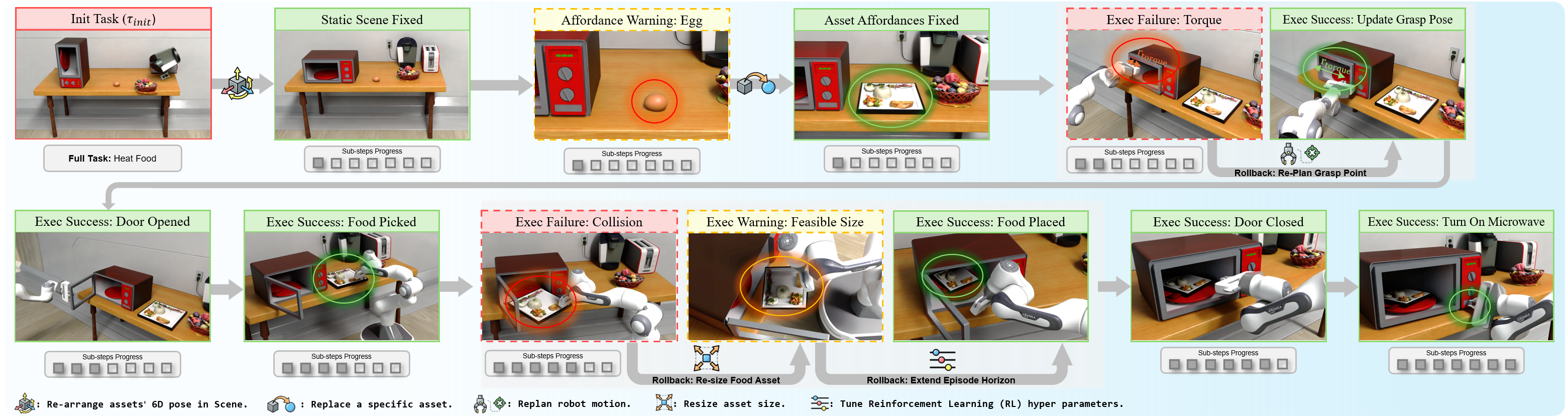}
    \caption{\textbf{Evolutionary Trajectory of the "Heat Food" Task.} 
  Starting from a chaotic initialization ($\tau_{init}$), FATE applies \textbf{Static Repair} to rearrange objects and correct affordances (swapping an egg for a meal). 
  During execution, the system autonomously detects and repairs dynamic failures: a grasp failure is resolved via \textbf{Motion Re-planning}, and an insertion collision is fixed via \textbf{Asset Resizing}. 
  The gray curved arrows illustrate our \textbf{Sub-step Rollback} mechanism, which invokes specific repair APIs (bottom icons) to iteratively refine the task into a successful policy.}
    \label{fig:auditor_dist}
\end{figure*}

\subsection{Phase II: Dynamic Alignment via Semantic Policy Gradient}
\label{subsec:dynamic_alignment}

While static alignment guarantees geometric validity, it cannot foresee failures arising from complex contact dynamics or controller instability. The second stage, $\mathcal{A}_{exec}$, addresses the \textit{Feasibility Gap} during runtime by ensuring the convergence of the policy specification $\Pi$.

\paragraph{Auditing: Monitoring Divergence.}
The In-step Auditor monitors the execution of primitive sequence $\{ \pi_k \}$. A task is flagged as a \textbf{Dynamic Divergence} if the agent fails to effect the desired state change within the allocated horizon or if the solver (MPC/RL) fails to find a valid solution. Common failure modes include high-torque grasp failures, collision-induced deadlocks, sparse reward stagnation, or latent geometric incompatibilities (e.g., tight insertion tolerances) that only manifest during physical interaction.

\paragraph{Repair: Execution Refinement.}
Dynamic failures imply that the policy specification $\Omega_\pi \subset \tau$ is ill-posed for the current scene. To resolve this, FATE triggers a rollback and synthesizes a dynamic residual $\Delta\tau_{dynamic}$ based on the failure trace $\xi_{fail}$. The task is updated via:
\begin{equation}
    \tau_{new} \leftarrow \tau_{old} \oplus \Delta\tau_{dynamic}
\end{equation}
Here, the residual $\Delta\tau_{dynamic}$ specifically targets the policy subspace $\Omega_\pi$. This update acts as a \textbf{Semantic Policy Gradient}, where the Auditor modifies either the physical scene (e.g., \textbf{Geometry Relaxation} by rescaling an object to ease insertion) or the solver logic (e.g., \textbf{Logic Tuning} by injecting a dense reward shaping term). Through these iterative updates, FATE pulls the divergent trajectory back into the feasible region $\Phi_\tau$.


\begin{proposition}[Convergence of Hierarchical Alignment]
\label{prop:hierarchical_convergence}
The sequential operation of static and dynamic alignment operators strictly increases the size of the feasible task set. Formally, the FATE operator $\mathcal{A} = \mathcal{A}_{exec} \circ \mathcal{A}_{perc}$ ensures
\begin{equation}
    | \Phi_{\mathcal{A}(\tau)} | \ge | \Phi_\tau |.
\end{equation}
\end{proposition}
The proof could be found in Appendix~\ref{sec:Proof_of_Convergence_for_Hierarchical_Alignment}. 
\paragraph{Implication.} Proposition~\ref{prop:hierarchical_convergence} guarantees that tasks with high semantic diversity but low structural stability are actively projected onto the feasibility manifold $\mu(\tau) > \delta_{min}$.


\section{Experiments}
\label{sec:experiments}

 Our experiments validate three core claims: (1) \textbf{Generative Capability}, ensuring the proposal distribution achieves high semantic and visual diversity; (2) \textbf{Feasibility Awareness}, assessing the embodied brain's fidelity in discriminating feasibility boundaries; and (3) \textbf{Refinement Efficacy}, quantifying the projection operators' success in bridging the feasibility gap.

\subsection{Experimental Setup}
\label{subsec:setup}

While framework-agnostic, we implement our system in NVIDIA Isaac Sim to leverage precise contact physics for the embodied brain's visual-statistical inputs. We evaluate on long-horizon tasks using the RidgebackFranka (7-DoF arm on a holonomic base), which requires coordinated navigation and manipulation, utilizing assets from Objaverse and PartNet-Mobility. We employ a hybrid execution strategy: \textbf{BIT*} for collision-free navigation and \textbf{Soft Actor-Critic (SAC)} for contact-rich manipulation. Both policy and Q-networks are MLPs with hidden layers of size $[256, 256]$, optimized via \textbf{Adam}.

\subsection{Generative Diversity}     
\label{sec:exp_diversity}

\textbf{Metrics and Baselines.} To quantify generative diversity, we employ semantic and visual metrics. We utilize \textbf{Self-BLEU-4} and \textbf{S-BERT Similarity} to measure linguistic redundancy, where lower scores indicate higher diversity. To capture structural variance, we assess \textbf{Visual Diversity} via ViT and CLIP embedding cosine similarity across scenes, ensuring morphological breadth. We benchmark against: (1) \textbf{Expert-Designed Benchmarks} (e.g., RLBench, ManiSkill2, Behavior-100) representing human-curated standards; and (2) \textbf{GenSim-V2}, a state-of-the-art open-loop baseline lacking feasibility feedback.

\begin{table*}[t]
\centering
\caption{\textbf{Generative Diversity Analysis.} We compare the semantic and visual diversity of tasks generated by FATE against expert-designed benchmarks and open-loop baselines. FATE achieves diversity comparable to human-curated benchmarks while ensuring physical feasibility.}
\label{tab:diversity_metrics}
\resizebox{\textwidth}{!}{%
\begin{tabular}{lccccc|c}
\toprule
 & \textbf{RLBench} & \textbf{ManiSkill2} & \textbf{Meta-World} & \textbf{Behavior-100} & \textbf{GenSim-V2} & \textbf{FATE (Ours)} \\ 
\midrule
\textbf{Number of Tasks} & 106 & 20 & 50 & 100 & 70 & \textbf{352} \\ 
\midrule
\multicolumn{7}{l}{\textit{Semantic Diversity (Language)}} \\
Task Desc. - Self-BLEU-4 ($\downarrow$) & 0.317 & 0.674 & 0.322 & 0.299 & 0.378 & \textbf{0.276} \\
Task Desc. - S-BERT Sim ($\downarrow$) & 0.200 & 0.194 & 0.263 & 0.210 & 0.288 & \textbf{0.158} \\ 
\midrule
\multicolumn{7}{l}{\textit{Visual \& Configural Diversity (Scene)}} \\
Scene Image - ViT Sim ($\downarrow$) & 0.375 & 0.332 & 0.517 & 0.389 & 0.717 & \textbf{0.189} \\
Scene Image - CLIP Sim ($\downarrow$) & 0.864 & 0.828 & 0.867 & 0.833 & 0.932 & \textbf{0.754} \\ 
\bottomrule
\end{tabular}%
}
\end{table*}

\textbf{Results.} 
Table~\ref{tab:diversity_metrics} suggests that: FATE maintains strong linguistic and visual diversity relative to open-loop generation. We attribute this to the closed-loop repair mechanism, which discourages repeatedly sampling "safe" but infeasible templates and instead pushes the generator to search for diverse task realizations that remain physically grounded.
\subsection{Task Feasibility}
\label{sec:exp_executability}
\textbf{Metrics and Baselines.}
To quantify the reduction of the feasibility gap, we employ the \textbf{Feasible Task Rate (FTR)}, evaluating whether a task $\tau$ is feasible (i.e., satisfies $\mu(\tau) > \delta_{min}$) under our verification pipeline. We decompose FTR into two sub-metrics: first, the \textbf{Scene Validity Rate (SVR)} measures perceptual alignment via BLIP-2 semantic coherence and geometric checks; second, the \textbf{Execution Validity Rate (EVR)} assesses dynamic feasibility using a Simulation Oracle to verify successful physical rollouts with MPC and RL solvers. We benchmark FATE against a \textbf{Vanilla (No Audit)} lower bound and the state-of-the-art \textbf{GenSim-V2}.
\begin{table}[h]
\centering
\caption{\textbf{Quantifying the Feasibility Gap.} We compare the validity of tasks generated by the vanilla baseline against FATE. We decompose the total yield (FTR) into static Scene Validity (SVR) and dynamic Execution Validity (EVR). }
\label{tab:executability_gap}
\resizebox{0.5\textwidth}{!}{%
\begin{tabular}{lccc}
\toprule
 & \textbf{Scene Validity (SVR)} & \textbf{Execution Validity (EVR)} & \textbf{Feasible Task Rate (FTR)} \\
 & \textit{(Static Validity) ($\uparrow$)} & \textit{(Dynamic Validity) ($\uparrow$)} & \textit{(End-to-End Yield) ($\uparrow$)} \\
\midrule
\multicolumn{4}{l}{\textit{Baselines}} \\
Vanilla (No Audit) & 41.5\% & 30.4\% & 12.6\% \\
GenSim-V2 & 58.2\% & 51.2\% & 29.8\% \\
\midrule
\multicolumn{4}{l}{\textit{FATE (Ours)}} \\
FATE-Static (Ante-Audit) & 93.8\% & 69.7\% & 65.4\% \\
\textbf{FATE (Full Pipeline)} & \textbf{97.5\%} & \textbf{94.5\%} & \textbf{92.1\%} \\
\bottomrule
\end{tabular}%
}
\end{table}
\textbf{Results.} Table~\ref{tab:executability_gap} highlights two key observations. First, static scene validity is necessary but not sufficient: many failures only emerge and discoverable during execution, so purely perceptual filtering leaves a large residual feasibility gap. Second, closing the loop with in-execution auditing and targeted repair substantially improves end-to-end yield, suggesting that feasibility is dominated by a small set of correctable misalignments (e.g., reachability, collisions, reward/horizon settings) rather than a lack of task diversity.
\subsection{Auditing Accuracy}
\label{sec:exp_auditor}
\textbf{Metrics and Baselines.}
The reliability of the FATE framework hinges on the Auditor's capacity to accurately approximate the feasibility indicator $\mathbb{I}[\mu(\tau) > \delta_{min}]$. The auditor detects infeasible tasks by checking for failures across three specific aspects: \textbf{Semantic Issues}, \textbf{Geometric Constraints}, and \textbf{Dynamic Issues}, as visually exemplified in Figure~\ref{fig:teaser}. To quantify this reliability, we measure \textbf{Auditor Efficacy} using the \textbf{F1-score} against a human-labeled ground truth dataset. This metric directly captures the discriminator's accuracy in identifying these diverse failure types while minimizing false positives.
\begin{table}[t]
\centering
\vspace{-.5cm}
\caption{\textbf{Diagnostic Performance of the Auditor.} We report the Precision, Recall, and F1-Score of the FATE Auditor across three error categories on a human-labeled test set. }
\label{tab:auditor_performance}
\resizebox{0.9\linewidth}{!}{%
\begin{tabular}{lccc}
\toprule
\textbf{Diagnostic Category} & \textbf{Precision} & \textbf{Recall} & \textbf{F1-Score} \\
\midrule
Semantic & 0.96 & 0.94 & 0.95 \\
Geometric & 0.88 & 0.85 & 0.86 \\
Dynamic & 0.85 & 0.81 & 0.83 \\
\midrule
\textbf{Overall} & \textbf{0.90} & \textbf{0.87} & \textbf{0.88} \\
\bottomrule
\end{tabular}%
}
\vspace{-.5cm}
\end{table}
\textbf{Results.} Table~\ref{tab:auditor_performance} confirms the Embodied Brain as a high-precision filter. The Ante-Auditor achieves F1-scores of \textbf{0.95} (\textit{Semantic}) and \textbf{0.86} (\textit{Geometric}), effectively pruning static hallucinations like affordance violations before simulation. The \textit{Dynamic} category proves more challenging (F1 \textbf{0.83}) due to the ambiguity between physical impossibility and algorithmic instability. However, a maintained precision of 0.85 ensures the In-step Auditor remains conservative, avoiding false negatives and guaranteeing the reliable feedback necessary for the convergence of the self-correcting loop.
\subsection{Repair Efficacy}
\label{sec:exp_repair}
\textbf{Metric Definition.}
To measure the contribution of the active repair loops, we define the \textbf{Repair Success Rate (RSR)}. This metric calculates the conditional probability that an initially infeasible task hypothesis (i.e., $\mu(\tau) \le \delta_{min}$) is successfully mapped into the feasible region after modification. We evaluate RSR across three distinct stages: \textbf{Ante-Repair} (Static scene adjustments), \textbf{Primitive-Repair} (Runtime logic corrections), and \textbf{RL-Repair} (Learning parameter tuning). A high RSR indicates that the system is not merely functioning as a filter (discarding infeasible data), but is actively projecting the proposal distribution onto the feasibility manifold.
\begin{table}[h]
\centering
\caption{\textbf{Repair Efficacy Analysis.} We report the Repair Success Rate (RSR) across three stages. The "Detected Failures" represents tasks identified as lying outside the manifold, and "Successful Repairs" denotes tasks successfully projected onto the manifold.}
\label{tab:repair_efficacy}
\resizebox{\linewidth}{!}{%
\begin{tabular}{lccc}
\toprule
\textbf{Repair Stage} & \textbf{Detected Failures (Input)} & \textbf{Successful Repairs (Output)} & \textbf{RSR (\%)} \\
\midrule
Ante-Repair (Static) & 215 & 201 & 93.5\% \\
Primitive-Repair (Runtime) & 42 & 39 & 92.9\% \\
RL-Repair (Learning) & 76 & 58 & 76.3\% \\
\midrule
\textbf{Aggregate} & \textbf{333} & \textbf{298} & \textbf{89.5\%} \\
\bottomrule
\end{tabular}%
}
\end{table}
\textbf{Results.} Table~\ref{tab:repair_efficacy} shows that most infeasible proposals are not fundamentally impossible. Most of these failures can be fixed by FATE eventually. The remaining hard cases largely stem from learning-induced brittleness, suggesting that improving the RL-side is the main bottleneck. Overall, these results support our claim that feasibility-aware generation is improving generated task qualities and executability.

Table~\ref{tab:repair_efficacy} quantifies the efficacy of the repair mechanisms. The high \textbf{Aggregate RSR (89.5\%)} reveals that the majority of generated proposals that initially have low feasibility (i.e., $\mu(\tau) \le \delta_{min}$) are essentially sound but misaligned. By sequentially applying the hierarchical adjustment operators—starting with \textit{Ante-Repair} to resolve static (scene-level) violations (93.5\% success), followed by \textit{In-step Repair} to address dynamic (execution-level) failures—FATE successfully recovers these tasks. Notably, even for the most challenging failures requiring \textit{RL-Repair}, the system salvages 76.3\% of cases. 
\begin{figure}[h]
    \centering
    \includegraphics[width=\linewidth]{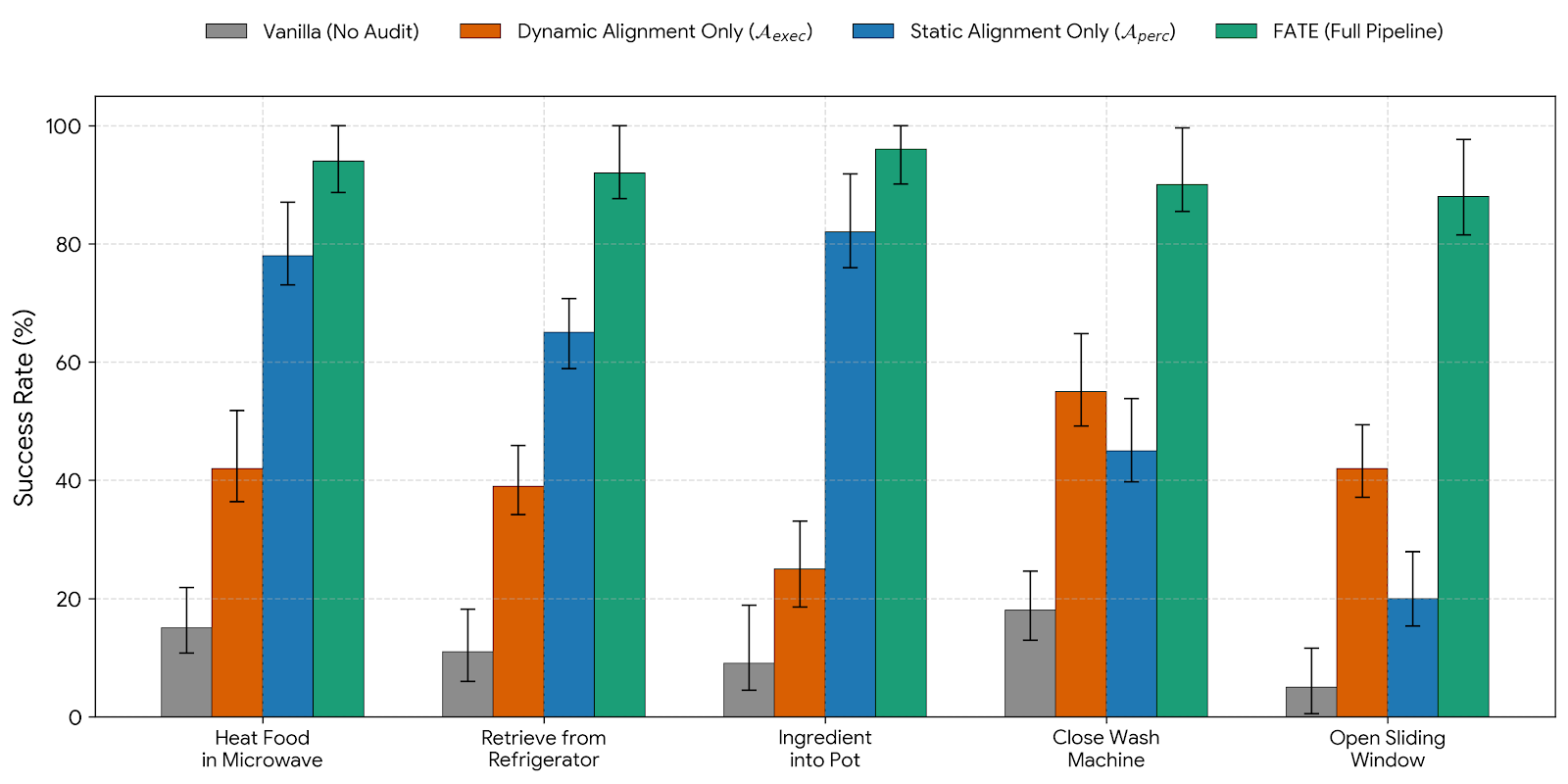}
    \caption{\textbf{Task-Wise Success Rate Analysis.} We evaluate performance across five tasks to disentangle the contributions of the alignment operators. The Static Alignment ($\mathcal{A}_{perc}$) module excels in geometry-sensitive tasks (e.g., \textit{Ingredient into Pot}), while the Dynamic Alignment ($\mathcal{A}_{exec}$) module is critical for contact-rich tasks (e.g., \textit{Sliding Window}). \textbf{FATE (Green)} combines both to achieve robust performance.}
    \label{fig:task_ablation}
\end{figure}
\subsection{Ablation Study}
\label{subsec:ablation}
To dissect the specific contributions of the alignment operators, we evaluate the Feasible Task Rate across five distinct semantic categories. This breakdown reveals the complementary nature of the static and dynamic phases.

For geometry-sensitive tasks such as \textit{''Ingredient into Pot''} or \textit{''Heat Food in Microwave,''} applying \textbf{Static Alignment ($\mathcal{A}_{perc}$)} alone yields significant gains, improving the yield by over $60\%$ compared to the vanilla baseline. In these scenarios, the primary failure mode is geometric incompatibility (e.g., the pot being smaller than the ingredient), which is efficiently resolved by the perceptual projection operator without requiring expensive simulation.

Conversely, for contact-rich tasks like \textit{''Open Sliding Window,''} Static Alignment is insufficient (20.6\%), performing worse than \textbf{Dynamic Alignment ($\mathcal{A}_{exec}$)} (41.8\%). This is because sliding mechanics involve complex friction and articulation dynamics that cannot be verified statically. The Dynamic Alignment operator captures these physical interaction failures, but the \textbf{Full FATE Pipeline} is required to achieve high reliability ($\approx 88\%$). By sequentially applying $\mathcal{A}_{exec} \circ \mathcal{A}_{perc}$, the system ensures that the scene is first geometrically feasible and then dynamically optimized, covering the diverse spectrum of open-ended tasks.

\section{Conclusion and Discussion}
In this work, we propose FATE, reformulating task synthesis as a closed-loop constraint optimization process. By integrating an Embodied Brain for hierarchical audit-and-repair, FATE aligns generated hypotheses with the feasibility manifold, mitigating physical hallucinations to produce a semantically diverse and rigorously grounded curriculum.

Despite significant gains, limitations persist. First, inherent VLM stochasticity means subtle inconsistencies may still bypass our iterative filters. Second, our framework focuses on geometric and kinematic constraints, leaving fine-grained dynamics (e.g., friction, damping) less rigorously modeled. Consequently, tasks dependent on precise contact physics may still suffer execution failures beyond the current repair.

Future work will address these by incorporating real-world failure logs to calibrate physical parameters against reality. Additionally, to mitigate computational latency and stochasticity, we plan to distill the auditing capabilities into a lightweight internal world model. This evolution aims to establish FATE as a self-correcting data engine, strictly grounded in both simulated and real-world physics.








\section*{Impact Statement}
This work advances open-ended robotic curriculum generation by closing the loop between language-driven task proposal and embodied feasibility verification. By substantially reducing execution failures and simulation waste, the approach can lower the compute cost of data generation and improve the reliability of downstream policy learning, potentially benefiting research reproducibility and accelerating development of assistive and service robots that operate in human environments. At the same time, methods that scale task synthesis may also lower the barrier to developing harmful or unauthorized automation, and feasibility in simulation does not guarantee safety in the real world—especially under unmodeled dynamics, embodiment mismatch, or auditing errors. Our experiments are confined to simulation, and failure modes related to contact dynamics and learned-controller brittleness may persist. We therefore position the system as a data engine and evaluation tool rather than a deployment-ready controller, and recommend responsible use with domain-specific validation, human oversight, and safety constraints.


\bibliography{example_paper}
\bibliographystyle{icml2026}

\newpage
\appendix
\onecolumn

\section{Theoretical Analysis of Iterative Repair}
\label{sec:appendix_proof}

In this section, we provide the detailed mathematical proof for \textbf{Proposition \ref{prop:convergence}}, establishing the linear convergence of the FATE repair mechanism. This analysis models the interaction between the Auditor (Embodied Brain) and the Task Manifold as a gradient-based dynamical system optimization.

\subsection{Setup and Definitions}

Let $\mathcal{T}$ denote the continuous task parameter space. To analyze the convergence properties, we define an \textbf{Infeasibility Potential} $\mathcal{J}: \mathcal{T} \to \mathbb{R}_{\ge 0}$ based on the feasibility measure $\mu(\tau)$ introduced in Definition \ref{def:feasibility_measure}:
\begin{equation}
    \mathcal{J}(\tau) \triangleq \max\left(0, \delta_{min} - \mu(\tau)\right)
\end{equation}
Our objective is to minimize this potential. The optimization proceeds via the update rule...

\subsection{Regularity Assumptions}

To guarantee convergence, we rely on standard assumptions in non-convex optimization, adapted for our "Approximate Oracle" (the LLM-based Auditor) setting.

\begin{assumption}[$L$-Smoothness]
    \label{ass:smoothness_detailed}
    The loss function $\mathcal{J}$ is differentiable and has an $L$-Lipschitz continuous gradient on the manifold. That is, for any $\tau, \tau'$, there exists a constant $L > 0$ such that:
    \begin{equation}
        \| \nabla \mathcal{J}(\tau) - \nabla \mathcal{J}(\tau') \| \le L \| \tau - \tau' \|
    \end{equation}
    This implies the standard quadratic upper bound inequality:
    \begin{equation}
        \label{eq:quadratic_bound}
        \mathcal{J}(\tau') \le \mathcal{J}(\tau) + \langle \nabla \mathcal{J}(\tau), \tau' - \tau \rangle + \frac{L}{2} \| \tau' - \tau \|^2
    \end{equation}
\end{assumption}

\begin{assumption}[Auditor Alignment Condition]
    \label{ass:alignment_detailed}
    The Auditor does not compute the exact gradient $\nabla \mathcal{J}(\tau_k)$ (which requires expensive Monte Carlo simulation). Instead, it provides a semantic direction $d_k$ that is strictly aligned with the true gradient. We assume there exist constants $c_1, c_2 > 0$ such that:
    \begin{equation}
        \langle d_k, \nabla \mathcal{J}(\tau_k) \rangle \ge c_1 \| \nabla \mathcal{J}(\tau_k) \|^2 \quad \text{(Sufficient Descent)}
    \end{equation}
    \begin{equation}
        \| d_k \| \le c_2 \| \nabla \mathcal{J}(\tau_k) \| \quad \text{(Bounded Magnitude)}
    \end{equation}
\end{assumption}

\begin{assumption}[Polyak-Łojasiewicz (PL) Condition]
    \label{ass:pl_detailed}
    Within the basin of attraction of the feasible manifold $\mathcal{F} = \{ \tau \mid \mathcal{J}(\tau) = 0 \}$, the loss satisfies the PL condition with parameter $\nu > 0$:
    \begin{equation}
        \frac{1}{2} \| \nabla \mathcal{J}(\tau) \|^2 \ge \nu \mathcal{J}(\tau)
    \end{equation}
\end{assumption}

\subsection{Proof of Linear Convergence (Proposition \ref{prop:convergence})}

\begin{proposition}[Linear Convergence of Iterative Repair]
\label{prop:convergence}
Under Assumptions~\ref{ass:smoothness_detailed}--\ref{ass:pl_detailed}, for a sufficiently small step size $\eta$, the iterative repair update $\tau_{k+1} = \tau_k - \eta d_k$ converges linearly to the feasible manifold $\mathcal{F} = \{\tau \mid \mu(\tau) \ge \delta_{min}\}$.
\end{proposition}

\begin{proof}
    We begin by substituting the update rule $\tau_{k+1} = \tau_k - \eta d_k$ into the quadratic upper bound \eqref{eq:quadratic_bound}:
    \begin{equation}
        \mathcal{J}(\tau_{k+1}) \le \mathcal{J}(\tau_k) + \langle \nabla \mathcal{J}(\tau_k), -\eta d_k \rangle + \frac{L}{2} \| -\eta d_k \|^2
    \end{equation}
    
    Rearranging the terms and applying the scalar properties of the inner product:
    \begin{equation}
        \mathcal{J}(\tau_{k+1}) \le \mathcal{J}(\tau_k) - \eta \langle \nabla \mathcal{J}(\tau_k), d_k \rangle + \frac{L \eta^2}{2} \| d_k \|^2
    \end{equation}
    
    Next, we apply the \textbf{Auditor Alignment Assumption} (\ref{ass:alignment_detailed}). We use the lower bound for the inner product and the upper bound for the norm of $d_k$:
    \begin{equation}
        \mathcal{J}(\tau_{k+1}) \le \mathcal{J}(\tau_k) - \eta c_1 \| \nabla \mathcal{J}(\tau_k) \|^2 + \frac{L \eta^2}{2} (c_2 \| \nabla \mathcal{J}(\tau_k) \|)^2
    \end{equation}
    
    Factor out $\| \nabla \mathcal{J}(\tau_k) \|^2$:
    \begin{equation}
        \mathcal{J}(\tau_{k+1}) \le \mathcal{J}(\tau_k) - \left( \eta c_1 - \frac{L \eta^2 c_2^2}{2} \right) \| \nabla \mathcal{J}(\tau_k) \|^2
    \end{equation}
    
    To ensure a decrease in loss, the coefficient of the gradient norm must be positive. This imposes a constraint on the step size $\eta$:
    \begin{equation}
        \eta c_1 - \frac{L \eta^2 c_2^2}{2} > 0 \implies \eta < \frac{2 c_1}{L c_2^2}
    \end{equation}
    
    Assuming $\eta$ satisfies this condition, we define the descent coefficient $\mathcal{C}_\eta = \eta c_1 - \frac{L \eta^2 c_2^2}{2} > 0$.
    Now, applying the \textbf{PL Condition} (\ref{ass:pl_detailed}), we substitute $\| \nabla \mathcal{J}(\tau_k) \|^2 \ge 2\nu \mathcal{J}(\tau_k)$:
    \begin{equation}
        \mathcal{J}(\tau_{k+1}) \le \mathcal{J}(\tau_k) - \mathcal{C}_\eta \cdot (2\nu \mathcal{J}(\tau_k))
    \end{equation}
    
    Simplifying the expression:
    \begin{equation}
        \mathcal{J}(\tau_{k+1}) \le \mathcal{J}(\tau_k) (1 - 2\nu \mathcal{C}_\eta)
    \end{equation}
    
    Let $\rho = 2\nu (\eta c_1 - \frac{L \eta^2 c_2^2}{2})$. For a sufficiently small step size $\eta$, we can ensure $0 < \rho < 1$. Thus, we obtain the linear recurrence relation:
    \begin{equation}
        \mathcal{J}(\tau_{k+1}) \le (1 - \rho) \mathcal{J}(\tau_k)
    \end{equation}
    
    By induction, after $k$ steps:
    \begin{equation}
        \mathcal{J}(\tau_k) \le (1 - \rho)^k \mathcal{J}(\tau_0)
    \end{equation}
    
    Since $(1-\rho) < 1$, $\lim_{k \to \infty} (1-\rho)^k = 0$, implying $\lim_{k \to \infty} \mathcal{J}(\tau_k) = 0$. Therefore, the task $\tau_k$ converges linearly to the feasible region where $\mu(\tau) \ge \delta_{min}$.
\end{proof}

\section{Proof of Convergence for Hierarchical Alignment}
\label{sec:Proof_of_Convergence_for_Hierarchical_Alignment}

In this section, we provide the detailed proof for \textbf{Proposition \ref{prop:hierarchical_convergence}}. We demonstrate that the FATE operator $\mathcal{A} = \mathcal{A}_{exec} \circ \mathcal{A}_{perc}$ strictly maintains or expands the volume of the feasible policy set.

\subsection{Measure-Theoretic Formulation}

Let $(\Omega_\pi, \Sigma, \rho_{init})$ be the measure space of policy parameters. For any task $\tau$, the feasible set is defined as:
\begin{equation}
    \Phi_\tau = \{ \theta \in \Omega_\pi \mid \mathbb{E}_{\xi \sim (\pi_\theta, \mathcal{S})} [\Psi_\mathcal{I}(\xi)] \ge 1 - \epsilon \}
\end{equation}
The feasibility volume is the measure $V(\tau) = \rho_{init}(\Phi_\tau)$. We analyze the monotonicity of $V(\cdot)$ through the two stages of alignment.

\subsection{Proof of Monotonicity}

\paragraph{Stage 1: Static Alignment ($\mathcal{A}_{perc}$).}
Let $\tau_{init} = (\mathcal{I}, \mathcal{S}, \Omega_\pi)$ be the initial task proposal.
We define the \textit{Static Validity Set} $\Theta_{static}(\mathcal{S}, \mathcal{I}) \subset \Omega_\pi$ as the set of policies that do not violate fundamental physical laws (e.g., trying to grasp an object that is kinematically unreachable).
\begin{lemma}[Static Necessity]
    For any policy $\theta$, if $\theta \in \Phi_\tau$ (is dynamically feasible), then $\theta \in \Theta_{static}$ (must be statically valid). Thus, $\Phi_\tau \subseteq \Theta_{static}$.
\end{lemma}
If the initial task has static conflicts (e.g., severe interpenetration), then $\Theta_{static} = \emptyset$, which implies $\Phi_{\tau_{init}} = \emptyset$ and $V(\tau_{init}) = 0$.
The Static Alignment operator $\mathcal{A}_{perc}$ applies a repair $\Delta_{static}$ to ensure static compatibility, thereby creating a non-empty static validity set $\Theta_{static}' \neq \emptyset$.
While this does not guarantee dynamic feasibility, it transforms the volume from a strictly zero state to a potentially positive state. Thus:
\begin{equation}
    V(\mathcal{A}_{perc}(\tau_{init})) \ge V(\tau_{init}) = 0
\end{equation}

\paragraph{Stage 2: Dynamic Alignment ($\mathcal{A}_{exec}$).}
Let $\tau_{static} = \mathcal{A}_{perc}(\tau_{init})$ be the statically valid task.
The Dynamic Alignment operator employs a "Rollback-and-Repair" mechanism. We model the execution of $\mathcal{A}_{exec}$ as generating a sequence of modified task/solver configurations $\mathbf{T} = \{ \tau^{(0)}, \tau^{(1)}, \dots, \tau^{(T)} \}$, where $\tau^{(0)} = \tau_{static}$.
Crucially, the FATE system accepts a task as "feasible" if the policy succeeds under *any* of the refined configurations found during the iterative process.
Let $\Phi^{(k)}$ be the feasible policy set associated with configuration $\tau^{(k)}$. The effective feasible set for the aligned task $\mathcal{A}_{exec}(\tau)$ is the union of feasible sets across the repair trajectory:
\begin{equation}
    \Phi_{total} = \bigcup_{k=0}^{T} \Phi^{(k)}
\end{equation}
Since the initial configuration corresponds to $k=0$, we have the trivial inclusion:
\begin{equation}
    \Phi^{(0)} \subseteq \bigcup_{k=0}^{T} \Phi^{(k)} = \Phi_{total}
\end{equation}
By the monotonicity of measures, if $A \subseteq B$, then $\rho_{init}(A) \le \rho_{init}(B)$. Therefore:
\begin{equation}
    V(\tau^{(0)}) \le \rho_{init}(\Phi_{total})
\end{equation}
Assuming the Auditor provides non-destructive updates (i.e., it does not purely shrink the feasible space of subsequent steps without adding new feasible regions), the effective solution space accessible to the system expands with each iteration.

\paragraph{Conclusion.}
Combining the results from Stage 1 and Stage 2:
\begin{equation}
    V(\mathcal{A}(\tau_{init})) = \rho_{init}(\Phi_{total}) \ge V(\tau_{static}) \ge V(\tau_{init})
\end{equation}
Consequently, the hierarchical alignment process guarantees that the measure of the feasible solution space is strictly non-decreasing, verifying Proposition \ref{prop:hierarchical_convergence}. \qed


\section{Implementation Details of the FATE Pipeline}
\label{app:implementation}

\begin{quote}
\textbf{Note:} This appendix details the constructive realization of the \textbf{Feasibility Alignment Operator} $\mathcal{A}$ defined in Section \ref{sec:methodology}. We describe how the theoretical projection steps ($\mathcal{A}_{perc}$ and $\mathcal{A}_{exec}$) are instantiated as concrete algorithm loops governed by the \textbf{Embodied Brain} (detailed in Appendix \ref{app:EmbodiedBrain}).
\end{quote}

\subsection{Initialization: Diverse Task Proposal}
\label{sub:init_proposal}

Before the alignment phases begin, we first sample from the generative distribution $\tau_{init} \sim P_{\mathcal{G}}$ (as defined in Eq. 3). The goal is to maximize semantic diversity before feasibility constraints are rigorously applied.

\paragraph{Semantic Candidate Retrieval.}
We utilize assets from the \texttt{Objaverse} \cite{deitke2023objaverse} and \texttt{PartNet-Mobility} \cite{Mo_2019_CVPR} datasets. To ensure the diversity of the proposal distribution, the Embodied Brain performs a broad semantic search based on high-level tags (e.g., \texttt{microwave}, \texttt{oven}, \texttt{food}). At this stage, strictly geometric filtering is relaxed to allow for novel object combinations.

\paragraph{Scene Hypothesis Generation.}
The system outputs a preliminary Scene Graph $\mathcal{S}_{init}$ and an Instruction $\mathcal{I}$. This is parsed into a simulation-ready USD (Universal Scene Description) loader script, initializing the task tuple $\tau_{init} = (\mathcal{I}, \mathcal{S}_{init}, \Omega_{\pi}^{init})$.

\subsection{Phase I: Static Alignment Operator ($\mathcal{A}_{perc}$)}
\label{sub:phase1_static}

This phase implements the \textbf{Perceptual Projection Operator} $\mathcal{A}_{perc}$ (Section \ref{subsec:static_alignment}). Its objective is to maximize the static feasibility proxy $\hat{\mu}_{static}$ by injecting a semantic residual $\Delta_{static}$.

\paragraph{Diagnostic Taxonomy (Estimating $\hat{\mu}_{static}$).}
The Ante-Auditor evaluates the task tuple against the four feasibility axes defined in Section \ref{subsec:static_alignment}:
\begin{itemize}
    \item \textbf{Affordance Matching:} Checks if semantic categories support the requested action (e.g., matching "heat" with "microwave").
    \item \textbf{Geometric Reachability:} Checks if the Scene Graph $\mathcal{S}$ contains objects placed outside the robot's workspace.
    \item \textbf{Physical Plausibility:} Checks for severe interpenetration or unstable placement in the initial configuration $\mathcal{S}_{init}$.
    \item \textbf{Morphological Compatibility:} Checks if object scales are compatible (e.g., "ingredient fits inside pot").
\end{itemize}

\paragraph{Repair Action Space (Implementing $\oplus \Delta_{static}$).}
The update step $\tau' \leftarrow \tau \oplus \Delta_{static}$ is realized through a set of atomic scene manipulation APIs (Table \ref{tab:static_api}). These primitives directly address the violations in $\hat{\mu}_{static}$.

\begin{table}[h]
\caption{\textbf{Static Repair API ($\mathcal{A}_{perc}$).} These primitives implement the discrete update $\Delta_{static}$ to resolve perceptual hallucinations defined in Section \ref{subsec:static_alignment}.}
\label{tab:static_api}
\begin{center}
\begin{small}
\begin{tabular}{p{0.35\linewidth} p{0.6\linewidth}}
\toprule
\textbf{API Signature} & \textbf{Functionality \& Theoretical Mapping} \\
\midrule
\texttt{SWAP\_ASSET(tgt\_id, query)} & Substitutes an object to maximize \textbf{Affordance Matching}. Used when $\Psi_{\mathcal{I}}$ cannot be satisfied by the current object semantics. \\
\midrule
\texttt{SPAWN\_ASSET(query, pose)} & Instantiates missing nodes in the Scene Graph $\mathcal{S}$. Used when the instruction $\mathcal{I}$ implies objects not present in $\mathcal{S}_{init}$. \\
\midrule
\texttt{REMOVE\_ASSET(tgt\_id)} & Prunes the Scene Graph. Used to remove clutter that violates \textbf{Geometric Reachability} or creates collision conflicts. \\
\midrule
\texttt{SET\_JOINT(tgt\_id, idx, val)} & Forces configuration state $q_j$ to a specific value. Used to resolve logical inconsistencies (e.g., "Open a door" that is already open). \\
\midrule
\texttt{TRANSFORM\_POSE(tgt\_id, $\Delta$)} & Applies a rigid body transform. Used to resolve interpenetration (\textbf{Physical Plausibility}). \\
\midrule
\texttt{RESCALE(tgt\_id, factor)} & Scales the physical manifold $M_j$. Directly addresses \textbf{Morphological Compatibility} (e.g., resizing an object to fit a container). \\
\bottomrule
\end{tabular}
\end{small}
\end{center}
\end{table}

\begin{algorithm}[tb]
   \caption{Static Alignment Loop ($\mathcal{A}_{perc}$)}
   \label{alg:static_repair}
\begin{algorithmic}
   \STATE {\bfseries Input:} Initial Task $\tau = (\mathcal{I}, \mathcal{S})$
   \STATE $\mathcal{S}_{curr} \leftarrow \mathcal{S}$
   \WHILE{True}
       \STATE \textcolor{gray}{// Estimate Static Feasibility Proxy $\hat{\mu}_{static}$}
       \STATE $Critique, IsValid \leftarrow \text{AnteAuditor}(\mathcal{I}, \mathcal{S}_{curr})$ 
       \IF{$IsValid$}
           \STATE \textbf{break}
       \ENDIF
       \STATE \textcolor{gray}{// Synthesize residual $\Delta_{static}$}
       \STATE $Repair\_Op \leftarrow \text{Brain.generate\_repair}(Critique)$
       \STATE \textcolor{gray}{// Apply projection $\mathcal{S} \leftarrow \mathcal{S} \oplus \Delta_{static}$}
       \STATE $\mathcal{S}_{curr} \leftarrow \text{Simulator}.apply(Repair\_Op, \mathcal{S}_{curr})$
       \IF{$MaxIter$ reached}
           \STATE \textbf{return} Failure
       \ENDIF
   \ENDWHILE
   \STATE \textbf{return} $\mathcal{S}_{curr}$
\end{algorithmic}
\end{algorithm}

\subsection{Phase II: Dynamic Alignment Operator ($\mathcal{A}_{exec}$)}
\label{sub:phase2_dynamic}

This phase instantiates the \textbf{Dynamic Alignment Operator} $\mathcal{A}_{exec}$ (Section \ref{subsec:dynamic_alignment}). It focuses on refining the Policy Specification $\Omega_\pi$ and Solver Hyperparameters $\theta$ to ensure the feasible policy set $\Phi_\tau$ is non-empty.

\paragraph{Hierarchical Decomposition \& Checkpointing.}
The task is decomposed into primitives $\Pi = \{\pi_1, \dots, \pi_K\}$. To support the iterative optimization of the \textbf{Semantic Policy Gradient} (Eq. 9), we employ a state checkpointing mechanism. This allows the system to rollback and apply the update $\theta^{(t+1)} \leftarrow \theta^{(t)} \oplus \nabla_{sem}$ without resetting the entire environment.

\paragraph{Dynamic Intervention API (Implementing $\nabla_{sem}$).}
When the In-Step Auditor detects a \textbf{Dynamic Divergence}, it synthesizes a repair command. For MPC primitives, repairs typically adjust geometric constraints (inflation radius). For RL primitives, the Brain actively synthesizes the \textbf{Reward Landscape} (Table \ref{tab:dynamic_api}).

\begin{table}[h]
\caption{\textbf{Dynamic Repair API ($\mathcal{A}_{exec}$).} These primitives instantiate the semantic policy gradient $\nabla_{sem}$, modifying the solver landscape to capture feasible solutions.}
\label{tab:dynamic_api}
\begin{center}
\begin{small}
\begin{tabular}{p{0.35\linewidth} p{0.6\linewidth}}
\toprule
\textbf{API Signature} & \textbf{Functionality \& Theoretical Mapping} \\
\midrule
\multicolumn{2}{l}{\textit{RL-Specific Optimizations (Policy Manifold Shaping)}} \\
\texttt{INJECT\_REWARD(func\_str)} & Overwrites the reward function with code synthesized by the Brain. This implements \textbf{Reward Landscape Synthesis}, allowing the VLM to reshape the optimization surface to guide exploration toward $\Phi_\tau$. \\
\texttt{SET\_RL\_HORIZON(steps)} & Adjusts the maximum episode length. This expands the temporal search space of $\Omega_\pi$, ensuring sufficient time for convergence. \\
\midrule
\multicolumn{2}{l}{\textit{Execution Dynamics \& Planner Tuning}} \\
\texttt{REPLAN\_PATH(inflation\_r)} & Modifies configuration space constraints for the MPC solver. Used to resolve collision-induced deadlocks. \\
\texttt{ADJUST\_GRASP(offset, vector)} & Modifies the end-effector approach pose. Used to correct high-torque grasp failures or singularities. \\
\texttt{TUNE\_IMPEDANCE(stiffness)} & Adjusts controller gains. Ensures stability during contact-rich interactions (e.g., opening heavy doors). \\
\texttt{SKIP\_SUBSTEP()} & Logically bypasses a primitive if the transition condition is already met (e.g., door opened by accidental contact). \\
\texttt{RESET\_ROBOT\_JOINT(cfg)} & Recovers from kinematic singularities, restoring the agent to a valid region of $\Omega_\pi$. \\
\bottomrule
\end{tabular}
\end{small}
\end{center}
\end{table}



\section{Embodied Brain: Architecture and Specialization}
\label{app:EmbodiedBrain}

In the FATE framework, the efficacy of the auditing mechanism relies fundamentally on the underlying model's ability to discern physical constraints from semantic hallucinations. We instantiate the Embodied Brain using a specialized version of RoboBrain 2.0 \citep{team2025robobrain}, a 7B parameter Vision-Language Model fine-tuned specifically for spatial reasoning and executable planning. This appendix outlines the model's architectural foundations, the composition of the training curriculum, and the quantitative validation of its embodied capabilities.

\subsection{Model Architecture}
The Auditor is built upon the RoboBrain 2.0 architecture, which derives from the Qwen2.5-VL foundation \citep{bai2023qwen}. The system comprises three primary components designed to bridge the gap between visual perception and logical reasoning. The perceptual front-end utilizes a Vision Transformer (ViT) encoder initialized from Qwen-VL's visual backbone, ensuring robust feature extraction from high-resolution inputs. These visual features are mapped into the linguistic embedding space via a multi-layer perceptron (MLP) projector, which facilitates precise modality alignment. The core reasoning engine is a Qwen2.5-7B Large Language Model, which processes the aligned visual tokens to generate semantic critiques and repair codes. To handle the computational demands of long-horizon robotic tasks, the training infrastructure employs deep learning optimization strategies including parameter sharding and attention acceleration, enabling the processing of video contexts up to 16k tokens.

\subsection{Training Data Formulation}
To equip the Embodied Brain with the necessary intuition for feasibility auditing, we leveraged a curated hybrid dataset of approximately 145,000 samples. The data distribution was explicitly engineered to address the action gap often observed in generalist VLMs, balancing three critical cognitive axes: Embodied Action, Spatial Perception, and Reasoning. The largest portion of the curriculum, comprising 45 percent of the dataset, focuses on Embodied Action. This segment aggregates data from benchmarks including RoboBench \citep{robobench}, EgoPlan-Bench2 \citep{egoplan}, and RoboSpatial-Home \citep{robospatial}, targeting capabilities in task decomposition, navigation, and affordance detection. An equal 45 percent share is dedicated to Spatial Perception, utilizing datasets like Blink \citep{blink}, OmniSpatial \citep{omnispatial}, and CVBench \citep{cvbench} to provide rigorous supervision for visual grounding and geometric understanding. The final 10 percent of the data preserves the model's high-level causal logic through the ERQA reasoning dataset \citep{erqa}. This strategic composition ensures that the Auditor does not merely label objects but understands their spatial relationships and interaction affordances within a physical scene.

\subsection{Empirical Validation}
The fine-tuned model was evaluated on the EmbodiedVerse benchmark \citep{he2025flagevalmm} to quantify its improvements over the baseline Qwen2.5-VL. The results demonstrate that the specialized training regimen significantly enhances the metrics most critical for the FATE auditing process. Specifically, the model achieved a 34.15 percent absolute improvement in Relative Shape estimation, representing an 80.01 percent relative gain over the baseline. This metric is vital for the Ante-Auditor's ability to detect morphological incompatibilities, such as an object being too large for a container. Furthermore, the model exhibited a 5.00 percent absolute improvement in Goal Decomposition. This enhancement validates the In-step Auditor's capacity to break down high-level instructions into executable primitives and identify logical breaks in the planning sequence. Crucially, these gains were achieved without catastrophic forgetting, as evidenced by a 2.17 percent improvement in general perception tasks like Counting. These performance metrics confirm that the Embodied Brain possesses the requisite spatial precision and planning foresight to serve as a reliable feasibility auditor within the FATE loop.


\end{document}